\documentclass[letterpaper]{article} 
\usepackage{aaai2027}  
\usepackage[hyphens]{url}  
\usepackage{graphicx} 
\usepackage{natbib}  
\usepackage{caption} 

\usepackage{subcaption}
\usepackage{array}
\usepackage{amsmath}
\usepackage{amssymb}
\usepackage{amsthm}
\usepackage{booktabs}
\usepackage{multirow}
\newtheorem{lemma}{Lemma}
\newcommand{\elementname}[1]{\emph{#1}}

\title{Element-Aware Group Learning\\for E-Commerce Image Generation}
\author{
    Jingtong Chen\equalcontrib,
    Jiahui Wang\equalcontrib,
    Xue Zhao\corresponding,
    ShaoGuo Liu\corresponding,
    Minghao Li
}
\affiliations{}

\begin{document}

\nocopyright
\maketitle

\begin{abstract}
Recent advances in image generation and editing have made prompt quality a key bottleneck for e-commerce creatives. Vision-language models (VLMs) can generate image-editing prompts from product images and metadata, but further improving their prompt-writing capabilities requires post-training with feedback from the generated images. Group Relative Policy Optimization (GRPO) is a natural framework for such outcome-level reward optimization. However, it assigns credit only at the full-prompt level, even though image quality often depends on specific design elements such as composition, background, and the presentation of selling points. Existing fine-grained credit assignment methods typically require step-level supervision or learned critics. To address this, we propose EAGLE-GRPO (\textbf{E}lement-\textbf{A}ware \textbf{G}roup \textbf{L}earning for \textbf{E}-Commerce Image Generation), which decomposes the group-centered reward over predefined elements. We cast element-level credit assignment as a kernel ridge regression problem and derive a closed-form solution, without additional rollouts or separate credit-assignment models. This yields interpretable per-element advantages and more precise policy updates. Experiments show that EAGLE-GRPO sustains performance gains over more training steps before plateauing and generates prompts that produce higher-quality e-commerce images than competitive VLM prompt-writing baselines.
\end{abstract}

\section{Introduction}

Modern image editors can produce high-quality e-commerce visuals, but the prompt itself spans preservation, layout, headline, scene, and other design decisions. VLMs are commonly used to enhance such prompts from a product image and metadata (e.g., title, description), then refine them via reward feedback. 
GRPO~\cite{shao2024deepseekmath} has emerged as a popular framework for such reward-driven VLM training, as it removes the need for a separate value model by normalizing rewards within sampled groups.
However, in standard GRPO each candidate receives one scalar image reward, and the group-normalized advantage is broadcast uniformly to all tokens. This coarse credit assignment conflates good and bad decisions, allowing weak fields to ride high rewards earned by strong ones.

\begin{figure}[htbp]
\centering
\begin{subfigure}[t]{0.49\columnwidth}
\centering
\includegraphics[width=\linewidth]{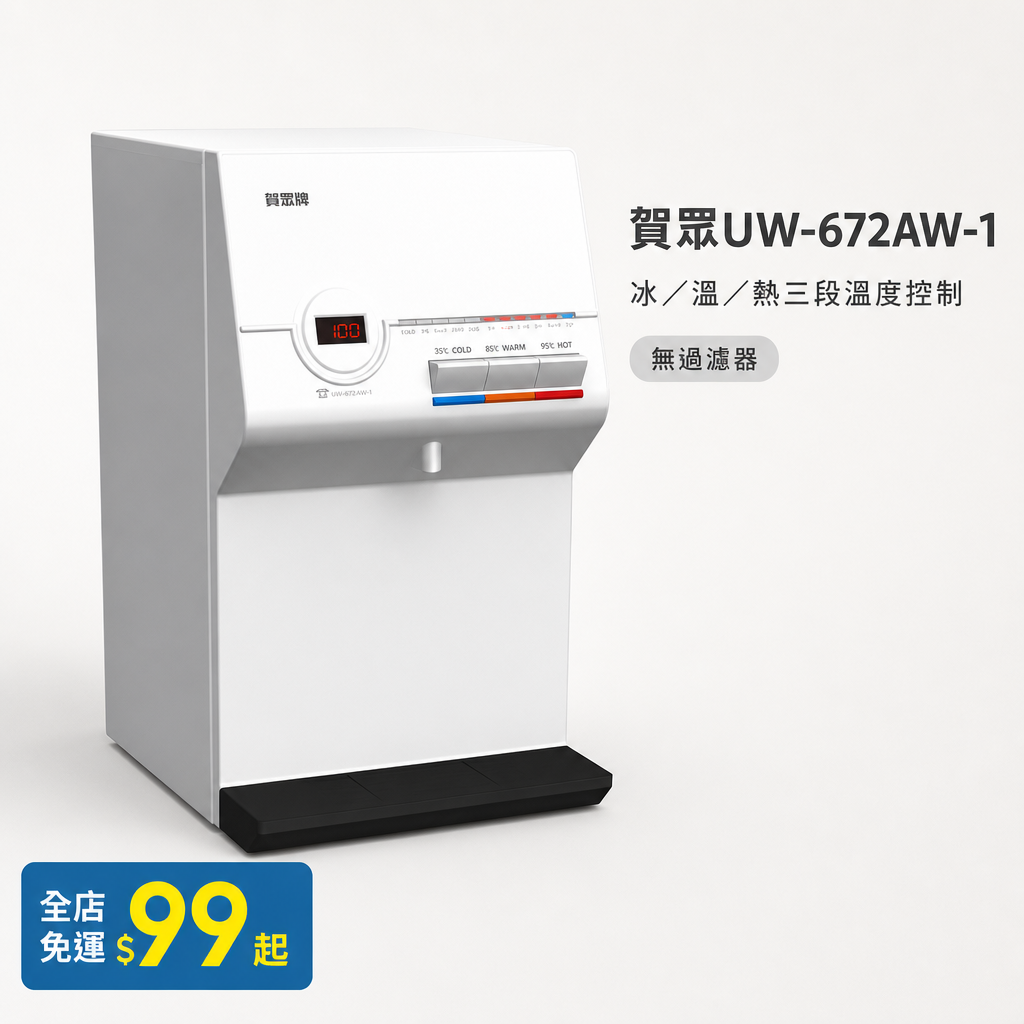}
\caption{Step 40 (reward=0.836)}
\end{subfigure}\hfill
\begin{subfigure}[t]{0.49\columnwidth}
\centering
\includegraphics[width=\linewidth]{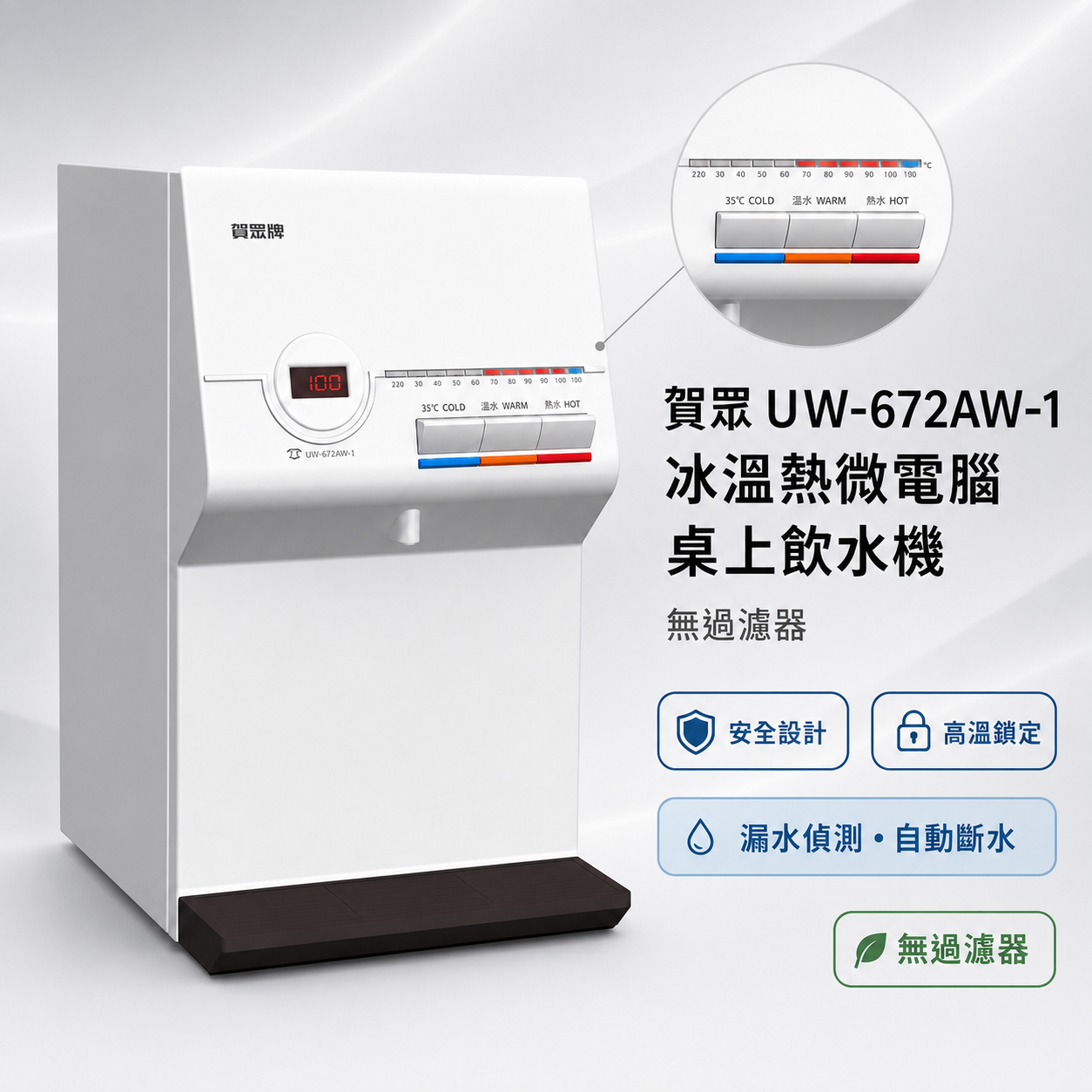}
\caption{Step 200 (reward=0.886)}
\end{subfigure}
\caption{Illustrative evolution of visual design elements during \textsc{Eagle}-GRPO training. Compared with the output at training step 40, the output at training step 200 includes a close-up supporting visual and multiple grounded feature callouts, yielding a more product-specific and informative presentation, consistent with its higher image reward.}
\label{fig:tw-prompt-evolution}
\end{figure}
\FloatBarrier



We argue that prompt optimization should operate at the level of design elements, which are visually entangled in the rendered image but textually separable in the structured prompt. Figure~\ref{fig:tw-prompt-evolution} illustrates how training refines individual visual elements by adding a supporting visual and grounded feature callouts while preserving the same product. This reframes the learning problem: given only a holistic image reward, how do we assign credit to individual elements?



A natural approach is Shapley-value attribution~\cite{shapley1953value,cao2025scar,li2026sharp}, which assigns each element's credit by comparing outcomes with and without that element. Yet this element-level variation is already present in a GRPO rollout group as a byproduct of sampling, since different rollouts differ in which elements they include and how they are realized.
Building on this free variation, we therefore propose \textsc{Eagle}-GRPO (\textbf{E}lement-\textbf{A}ware \textbf{G}roup \textbf{L}earning for \textbf{E}-Commerce Image Generation), which assigns per-element credit from the reward variation already present in a GRPO rollout group. It parses each prompt into tagged design elements, pools their token spans into embeddings that serve as kernel features, and decomposes the centered image reward into per-element contributions via a closed-form kernel-ridge solution. The resulting credits are mapped back to token spans through a residual-preserving transformation, such that the token average equals the standard GRPO advantage (\textbf{Lemma~\ref{lem:preservation}}). 

The main contributions are:
\begin{itemize}
   \item We introduce a kernelized element-credit mechanism for GRPO that decomposes image rewards into per-element credits via a closed-form kernel-ridge solution, requiring no additional counterfactual rollouts.
    \item The element-credit design preserves the standard GRPO signal by construction: averaging the token-level element advantages over all tokens exactly recovers the group-normalized advantage, ensuring that credit is redistributed across tokens within each sample while preserving the per-sample mean advantage (Lemma~\ref{lem:preservation}).
    \item We evaluate against competitive prompt-writing baselines with GPT-Image-2, and FLUX.2~[klein] as image editors; \textsc{Eagle}-GRPO outperforms Standard GRPO on both editors and is preferred by human raters in 62\% of blind pairwise comparisons.
\end{itemize}

\section{Related Work}

\paragraph{Fine-Grained Credit Assignment for GRPO.}
Group Relative Policy Optimization (GRPO)~\cite{shao2024deepseekmath} and variants such as DAPO~\cite{yu2025dapo} normalize rewards within sampled groups but assign the same advantage to every token. Recent efforts to refine the credit itself include critic-free token reweighting via intrinsic proxies such as entropy or logit confidence~\cite{tan2025gtpo,he2026eapo}, critic-based methods that reintroduce learned value functions~\cite{schulman2016highdimensional}, and process reward models that densify the signal via step-level supervision~\cite{lightman2023lets,setlur2025rewarding,khalifa2025thinkprm,zheng2025survey}. Yet none of these refinements target the semantic structure of structured prompts: in our setting, the natural unit of credit is the prompt element (headline, scene, lighting, etc) rather than the token. \textsc{Eagle}-GRPO takes a different route: it parses the prompt into tagged elements and assigns each element a credit derived from the reward variation already present in a GRPO rollout group, with no trained value model or external annotation.

\paragraph{Reward Attribution Strategies.}
A separate line of work studies how to attribute a single outcome-level reward to finer output units. Counterfactual methods evaluate outcomes under different configurations and assign credit via cooperative-game solutions such as the Shapley value~\cite{shapley1953value}. For example, SCAR~\cite{cao2025scar} masks token spans and queries a reward model per configuration, SHARP~\cite{li2026sharp} uses leave-one-out evaluation for multi-agent tool-use, and SAVOIR~\cite{feng2026savoir} applies expected-utility Shapley values to multi-turn dialogue. These methods are principled, but they require interventional sampling beyond the rollout group, since each configuration demands an additional generation and reward query.

\textsc{Eagle}-GRPO takes the correlational route instead: rather than constructing counterfactual configurations, it attributes credit from reward variation already present across candidates. This attribution is formulated as an additive regression over element features. Its closed-form kernel-ridge decomposition is an instance of additive RKHS regression~\cite{wahba1990spline,gu2002smoothing,kandasamy2016salsa}, where each element defines a reproducing kernel and the centered reward is fit by an additive model regularized per component via the representer theorem~\cite{kimoneldorf1971representer}. A residual-preserving transformation then maps these credits back to token spans, such that the token-level average exactly recovers the standard GRPO advantage.
Prior GRPO-based credit assignment has focused on refining credit along tokens, steps, or segments; decomposition over structured, semantically tagged prompt elements remains largely unexplored. \textsc{Eagle}-GRPO addresses this setting.
 
\paragraph{Prompt Writing for Image Generation.}
Prior prompt-writing work uses three reward paradigms. Promptist~\cite{hao2023optimizing} trains a writer with PPO~\cite{schulman2017ppo} on external aesthetic and relevance scores. PromptEnhancer~\cite{wang2025promptenhancer} relies on a separate alignment evaluator to score chain-of-thought rewrites. Self-Rewarding LVLM~\cite{yang2025selfrewarding} unifies the writer and judge in a single VLM, iterating via self-generated preference signals.
All optimize the prompt as a single string against a scalar reward; none exploits the fact that a low image score often stems from one element rather than the entire prompt. \textsc{Eagle}-GRPO instead defines a structured prompt schema that decomposes the prompt into tag-delimited fields with well-defined token spans, enabling per-element credit assignment.

\begin{figure*}[h]
\centering
\includegraphics[width=\textwidth]{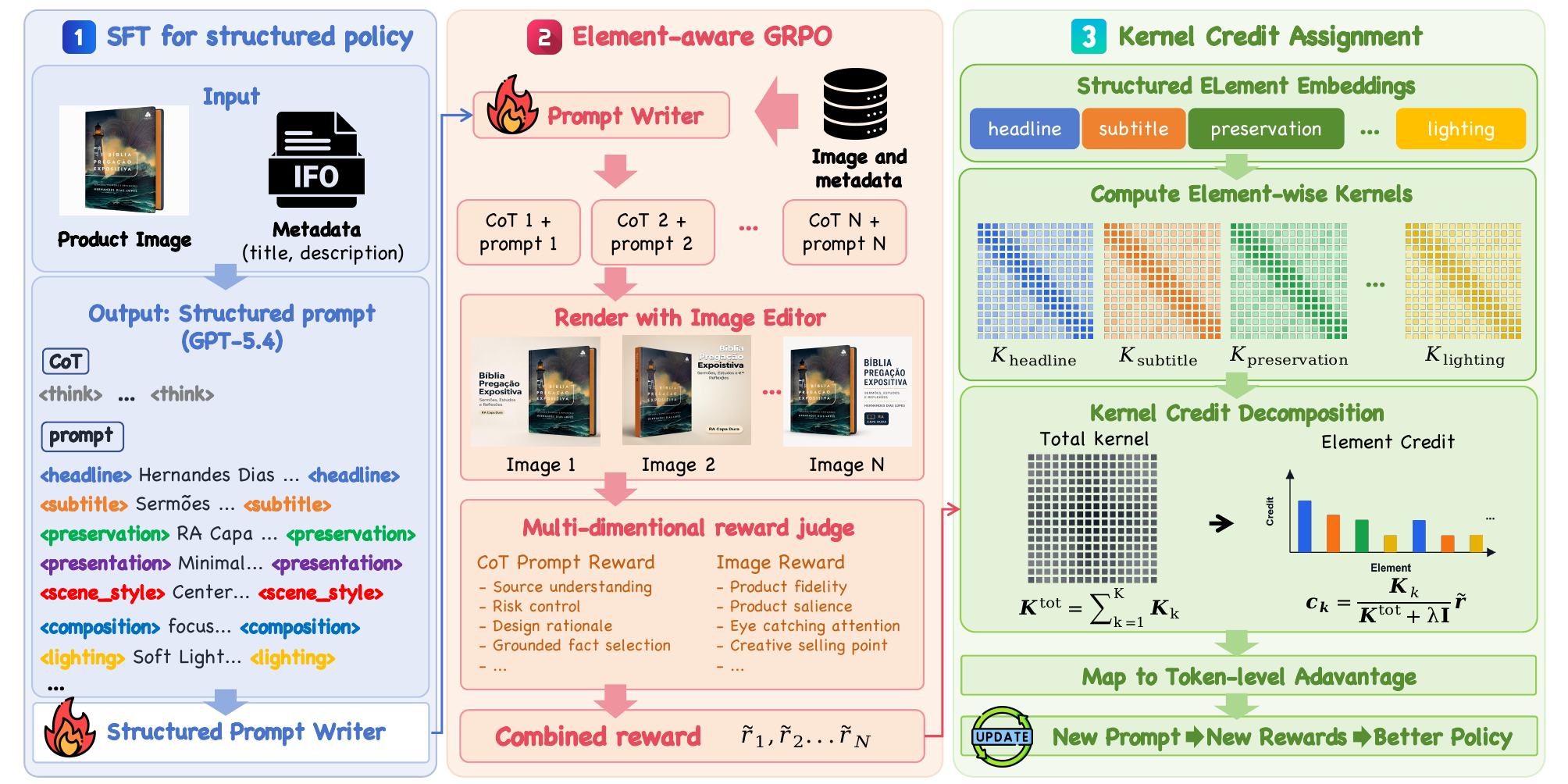}
\caption{Overview of \textsc{Eagle}-GRPO. \textbf{Stage 1: structured SFT}
teaches the prompt writer to analyze the product image and metadata, and produce an structured prompt with 18 tagged elements.
\textbf{Stage 2: element-aware GRPO} samples $N$ structured prompts, renders
them with an image editor, and obtains prompt- and image-level scores from a
multidimensional reward judge. \textbf{Stage 3: kernel credit assignment}
constructs element embeddings, builds an element-wise kernel $K_k$
for each element $k$, and forms the total kernel
$K^{tot}$. Here, $\tilde{\mathbf r}$ is the group-centered
combined reward vector, and
$\mathbf c_k$ is the reward
contribution attributed to element $k$. The element credits are mapped to token-level advantages and used to update the prompt policy.}
\label{fig:pipeline}
\end{figure*}

\section{Methodology}
\subsection{Prompt Element Schema}
E-commerce prompt writing combines product-preservation constraints with
commercial presentation decisions such as headline, layout, scene, lighting, etc. 
To make these decisions individually creditable, we structure each prompt as a
tag-delimited schema: every design decision occupies a dedicated field delimited
by atomic tags (e.g., \texttt{<headline>}...\texttt{<headline>}). Each such
field defines one \emph{element}, the atomic unit of credit assignment.
Given a product image $I_{\mathrm{prod}}$ and product metadata $m$, 
\begin{equation}
x=(I_{\mathrm{prod}},m),
\end{equation}
a VLM policy
first produces a reasoning trace $z^{\mathrm{cot}}$ that analyzes the product and
plans the design, then emits a schema-delimited prompt:
\begin{equation}
y=(z^{\mathrm{cot}},z), \qquad z=(z_1,\ldots,z_K),
\end{equation}
where $z$ is the full structured prompt and each $z_k$ is the $k$-th element. 
The schema gives each design decision a stable boundary tag, turning free-form
prompt into a set of creditable fields without an external parser. Table~\ref{tab:prompt-schema-elements} lists the 18 elements. Required and optional
fields capture creative design decisions; constraint-only fields
(preservation, negative) enforce faithfulness and safety boundaries. 
We use a strong VLM to generate high-quality image-editing prompts containing self-designed tags. We then supervised-fine-tune (SFT) a student VLM on this data to produce the schema reliably. 
Each schema tag is added to the tokenizer as an atomic delimiter for reliable
field parsing and hidden-state pooling. Figure~\ref{fig:pipeline} illustrates the full workflow.


\subsection{Element-Aware GRPO}
After SFT, for each training example $x$, the policy
$\pi_{\theta_{\mathrm{old}}}$ samples a group of $N$ outputs
$\{y_1,\ldots,y_N\}$, where
$y_i\sim\pi_{\theta_{\mathrm{old}}}(\cdot\mid x)$.
Writing $y_i=(z_i^{\mathrm{cot}},z_i)$, the structured prompt $z_i$ is rendered
by an editor $G$ and judged by a VLM reward model $R_{\mathrm{image}}$
(a strong off-the-shelf VLM, Gemini-2.5-Flash in our experiments):
\begin{equation}
I_{\mathrm{out},i}=G(I_{\mathrm{prod}},z_i), \qquad
r_i=R_{\mathrm{image}}(I_{\mathrm{prod}},I_{\mathrm{out},i},m).
\end{equation}
Standard GRPO normalizes rewards within the group:
\begin{equation}
A_{\mathrm{image},i}
=\frac{r_i-\mu}{\sigma+\epsilon_r},
\label{eq:standard-image-advantage}
\end{equation}
where $\mu$ and $\sigma$ are the group reward mean and standard deviation.
Broadcasting this scalar to all loss-bearing prompt tokens cannot distinguish which structured decisions made a rollout successful. 
\textsc{Eagle}-GRPO instead differentiates advantage across schema elements:
constraint-only fields (preservation, negative) receive the standard
image-level advantage $A_\mathrm{image}$
 ; for non-constraint fields,
the kernel module produces element-conditioned token credit $A_{\mathrm{elem},i,t}$. The two are blended as
\begin{equation}
A_{\mathrm{mix},i,t}
=(1-\eta)A_{\mathrm{image},i}
+\eta A_{\mathrm{elem},i,t}.
\label{mixA}
\end{equation}
where $\eta \in [0,1]$ controls the blend between image-level and element-level advantage.
By design, the token average of $A_{\mathrm{elem}}$ recovers 
$A_{\mathrm{image}}$, so element credit only redistributes within a rollout without changing its image-level mean.

\begin{table}[t]
\centering
\scriptsize
\setlength{\tabcolsep}{2.2pt}
\begin{tabular*}{\columnwidth}{@{\extracolsep{\fill}}p{0.42\columnwidth}p{0.48\columnwidth}@{}}
\toprule
\textbf{Field} & \textbf{Function} \\
\midrule
\multicolumn{2}{l}{\textbf{Constraint-only (not credited)}} \\
\emph{preservation\_constraints} & Preserve identity and geometry \\
\emph{negative\_constraints} & Prohibit unsupported edits \\
\midrule
\multicolumn{2}{l}{\textbf{Required (always credited)}} \\
\emph{product\_presentation} & Hero treatment \\
\emph{headline} & Primary message \\
\emph{composition\_layout} & Product / text arrangement \\
\emph{scene\_style} & Background and mood \\
\emph{lighting\_rendering} & Lighting and finish \\
\midrule
\multicolumn{2}{l}{\textbf{Optional (credited when present)}} \\
\emph{subtitle} & Secondary hierarchy \\
\emph{feature\_callouts} & Visualize benefits \\
\emph{specification\_information} & Specifications \\
\emph{quantity\_bundle\_information} & Count or bundle description \\
\emph{variant\_information} & Product variants \\
\emph{price\_promotion\_information} & Promotion \\
\emph{seller\_logistics\_information} & Seller / delivery facts \\
\emph{trust\_information} & Trust cues \\
\emph{usage\_information} & Use scenario \\
\emph{comparison\_visual} & Comparison \\
\emph{supporting\_visuals} & Insets or detail views \\
\bottomrule
\end{tabular*}
\caption{Structured prompt elements and credit rules. Constraint fields are not credited; required fields are always credited; optional fields are credited only when non-empty.}
\label{tab:prompt-schema-elements}
\end{table}

\subsection{Element-Level Credit Decomposition }
For every rollout, we construct a fixed 18-field representation, with kernel attribution applied to the 16 non-constraint fields. A non-empty field is represented by its pooled token embedding, whereas an absent optional field is represented by the zero vector:

\begin{equation}
\mathbf{v}_{i,k} =
\begin{cases}
\mathrm{LayerNorm}\!\left(\dfrac{1}{|S_{i,k}|}
\displaystyle\sum_{t \in S_{i,k}} \mathbf{h}_{i,t}\right), & |S_{i,k}|>0,\\[6pt]
\mathbf{0}, & |S_{i,k}|=0,
\end{cases}
\end{equation}
where $S_{i,k}$ denotes the set of token positions belonging to element $k$
in rollout $i$, and $\mathbf{h}_{i,t}$ denotes the last-layer hidden state computed by
$\pi_{\theta_{\mathrm{old}}}$ at rollout time. Consequently, each field kernel jointly captures field selection
(presence versus absence) and semantic variation among its non-empty
realizations. Credit is inferred from within-group variation, so embeddings are
centered within each rollout group:

\begin{equation}
\tilde{\mathbf{v}}_{i,k} = \mathbf{v}_{i,k} - \frac{1}{N}\sum_{j=1}^{N} \mathbf{v}_{j,k}.
\end{equation}
Within each group, the element kernel is
\begin{equation}
\mathbf{K}_k(i,j) = \tilde{\mathbf{v}}_{i,k}^{\top}\tilde{\mathbf{v}}_{j,k}.
\end{equation}
Let $\tilde r_i = r_i - \mu$ denote the centered image reward for rollout $i$,
where $\mu$ is the group mean, and let
$\tilde{\mathbf r} = [\tilde r_1,\ldots,\tilde r_N]^\top \in \mathbb{R}^N$
denote the vector of centered rewards within the group. We estimate an
element-contribution vector $\mathbf{c}_k \in \mathbb{R}^N$ for each element
$k$ by
\begin{equation}
\min_{\{\mathbf{c}_{k}\}_{k=1}^{K}} \quad
\left\| \tilde{\mathbf r}- \sum_{k=1}^{K} \mathbf{c}_{k} \right\|_2^2
+ \lambda \sum_{k=1}^{K}
\mathbf{c}_{k}^{\top}
(\mathbf{K}_{k} + \epsilon_K \mathbf{I})^{-1}
\mathbf{c}_{k}.
\end{equation}
The regularizer restricts element $k$ to claim credit only along directions
supported by its own variation within the group. The solution takes the closed
form $\mathbf{c}_{k} = \mathbf{K}_{k}\boldsymbol{\omega}$. With $\mathbf{K}^{\mathrm{tot}} = \sum_k \mathbf{K}_{k}$,
\begin{equation}
\boldsymbol{\omega} = (\mathbf{K}^{\mathrm{tot}} + \lambda \mathbf{I})^{-1}\tilde{\mathbf r},
\end{equation}
and therefore
\begin{equation}
\mathbf{c}_k = \mathbf{K}_k \boldsymbol{\omega} = \mathbf{K}_k (\mathbf{K}^{\mathrm{tot}} + \lambda \mathbf{I})^{-1} \tilde{\mathbf r}.
\end{equation}
Let $c_{i,k}=[\mathbf{c}_k]_i$ denote the contribution of element $k$
to rollout $i$. The portion of $\tilde r_i$ not attributed to any field is kept as a residual, 
\begin{equation}
\nu_i = \tilde r_i - \sum_{k=1}^{K} c_{i,k}.
\label{eq:residual}
\end{equation}

In practice, we apply conservative identifiability checks: near-zero reward variance or ill-conditioned kernels fall back to coarse image credit; low-energy element kernels receive zero contribution; highly aligned kernels are down-weighted. The decomposition is computed independently per group and detached from the policy gradient.

\subsection{Token-Level Advantage Mapping}
Each element has a token span in the final prompt. The element credit $c_{i,k}$ and residual $\nu_i$ are in reward space; to make them comparable with the standard GRPO advantage (Eq.~\ref{eq:standard-image-advantage}), we rescale them by the group standard deviation $\sigma$, obtaining the element-level advantage $Q_{i,k}$ and the per-rollout residual advantage $U_i$:
\begin{equation}
Q_{i,k} = \frac{c_{i,k}}{\sigma + \epsilon_r}, \quad
U_i = \frac{\nu_i}{\sigma + \epsilon_r}.
\label{eq:Q-U-def}
\end{equation}
These are then distributed across tokens within each element. For each element $k$ of $i-$th rollout, let $L_{i,k}$ be the number of loss-bearing tokens; $T_{i} = \sum_k L_{i,k}$ is the total count of loss-bearing tokens in $i-$th rollout. For a token $t$, let $k(t)$ denote the element containing $t$; the element-aware advantage is:

\begin{equation}
A_{\mathrm{elem},i,t} = U_i + \frac{T_i \cdot Q_{i,k(t)}}{L_{i,k(t)}}.
\label{eq:element-aware-advantage}
\end{equation}
Although $A_{\mathrm{elem},i,t}$ assigns different advantages to different tokens, its mean over all loss-bearing tokens in a rollout exactly recovers the rollout-level advantage $A_{\mathrm{image},i}$:

\begin{lemma}[Advantage Preservation]
\label{lem:preservation}
The $T_{i}$-token average of $A_{\mathrm{elem},i,t}$ over the loss-bearing tokens of candidate $i$ equals the original group-normalized GRPO advantage $A_{\mathrm{image},i}$, as in Eq.~\eqref{eq:standard-image-advantage}. That is,
\begin{equation}
\frac{1}{T_i} \sum_{t=1}^{T_i} A_{\mathrm{elem},i,t}
= A_{\mathrm{image},i}.
\end{equation}
\end{lemma}
\noindent\textit{Proof.}
From the residual definition Eq.~\eqref{eq:residual}, $\sum_k c_{i,k} + \nu_i = \tilde r_i$.
Dividing by $\sigma+\epsilon_r$ and applying Eq.~\eqref{eq:Q-U-def} yields
\begin{equation}
\sum_k Q_{i,k} + U_i = A_{\mathrm{image},i}.
\end{equation}
By construction of $A_{\mathrm{elem},i,t}$ in Eq.~\eqref{eq:element-aware-advantage}, its $T_i$-token average gives the same sum:
\begin{equation}
\frac{1}{T_i}\sum_{t=1}^{T_i} A_{\mathrm{elem},i,t}
= U_i + \sum_k Q_{i,k}
= A_{\mathrm{image},i}.
\end{equation}
%
%
Lemma~1 shows that element credit only redistributes advantage within a completion; it preserves the rollout-level image signal. For stability, we mix it with the standard image-level advantage as defined in Eq.~\eqref{mixA}.

\subsection{Training Objective}
The final GRPO objective uses $A_{\mathrm{mix},i,t}$ for all loss-bearing tokens; on constraint-only fields (preservation, negative) we set $\eta=0$, reducing $A_{\mathrm{mix}}$ to the standard image-level advantage $A_{\mathrm{image}}$.

\begin{equation}
\begin{aligned}
\mathcal{J}(\theta)
= {} & \mathbb{E}_{i,t}\!\left[
\min\!\left(
\rho_{i,t} A_{\mathrm{mix},i,t},
\bar{\rho}_{i,t} A_{\mathrm{mix},i,t}
\right)
\right] \\
& - \beta\, \mathrm{KL}\!\left[\pi_\theta \| \pi_{\mathrm{ref}}\right],
\end{aligned}
\end{equation}

where $\rho_{i,t} = \pi_\theta(y_{i,t} \mid x,y_{i<t} ) / \pi_{\theta_{\mathrm{old}}}(y_{i,t} \mid x,y_{i<t})$ is the importance ratio and $\bar{\rho}_{i,t}=\mathrm{clip}(\rho_{i,t},1{-}\epsilon_{\mathrm{clip}},1{+}\epsilon_{\mathrm{clip}})$. The element-credit decomposition contributes only through $A_{\mathrm{mix}}$ and introduces no additional learnable parameters.

\section{Experiments}
\label{sec:experiments}

We first describe the training and evaluation setup, and then compare \textsc{Eagle}-GRPO with strong prompt-writing baselines through benchmark and human evaluation.

\subsection{Experimental Setup}

\paragraph{Supervised initialization.}
We initialize the prompt writer from Qwen3-VL-8B-Instruct. The model is supervised-fine-tuned on 24,790 teacher responses from 6,240 products. For each product, the teacher generates up to four prompt profiles: sparse, balanced, visually expressive, and commercially dense. This exposes the model to both selective omission and detailed, grounded design choices.

Each response contains a short reasoning summary and an 18-field structured prompt. We add each schema tag to the tokenizer as a separate token and train its embedding and LM-head rows together with the SFT LoRA adapter. All other parameters remain frozen. The added tags provide explicit token-level boundaries for element parsing and hidden-state pooling.

\paragraph{GRPO training.}
 

Starting from the SFT-trained prompt writer, we train two editor-specific \textsc{Eagle}-GRPO policies, one for GPT-Image-2 and one for FLUX.2~[klein]. 
We merge the SFT LoRA weights into the base model to preserve the prompt schema, then add a fresh LoRA adapter for GRPO. Sampled prompts are rendered by the corresponding editor and scored by Gemini-2.5-Flash using the benchmark rubric (6 hard gates and 11 quality dimensions, Table~\ref{tab:judge-criteria}); failing any hard gate can zero the reward. For training stability, we adopt the DAPO objective~\cite{yu2025dapo} (with $\beta=0$ and truncation masking) for all GRPO training. Table~\ref{tab:implementation} summarizes the main hyperparameters.


\begin{table}[t]
\centering
\scriptsize
\setlength{\tabcolsep}{3pt}
\begin{tabular*}{\columnwidth}{@{\extracolsep{\fill}}p{0.31\columnwidth}p{0.28\columnwidth}p{0.31\columnwidth}@{}}
\toprule
\textbf{Setting} & \textbf{SFT} & \textbf{GRPO} \\
\midrule
Base model & Qwen3-VL-8B & SFT-ed Prompt Writer\\
Trainable modules & LoRA + new tag rows & LoRA \\
LoRA $(r,\alpha)$ & $(16,32)$ & $(16,32)$ \\
Learning rate & $2\!\times\!10^{-5}$ & $1\!\times\!10^{-6}$ \\
Optimizer/objective & AdamW & DAPO \\
LR schedule & cosine & constant \\
Effective batch & 8 responses & 80 completions \\
Precision & BF16 & BF16 \\
Rollouts $N$ & -- & 20 per product \\
\bottomrule
\end{tabular*}
\caption{Training hyperparameters for supervised initialization and \textsc{Eagle}-GRPO.}
\label{tab:implementation}
\end{table}


\begin{table}[htbp]
\centering
\scriptsize
\setlength{\tabcolsep}{3pt}
\begin{tabular*}{\columnwidth}{@{\extracolsep{\fill}}p{0.18\columnwidth}p{0.36\columnwidth}p{0.36\columnwidth}@{}}
\toprule
\textbf{Type} & \multicolumn{2}{c}{\textbf{Criteria}} \\
\midrule
\textbf{Hard gates} 
& \begin{tabular}[t]{@{}l@{}}
Product identity \\
Product count \\
Variant consistency
\end{tabular}
& \begin{tabular}[t]{@{}l@{}}
Logo/package text \\
Unsupported claims \\
Severe visual damage
\end{tabular} \\
\textbf{Quality (1--10)} 
& \begin{tabular}[t]{@{}l@{}}
Product fidelity \\
Product salience \\
Eye-catching attention \\
Creative selling point \\
Color harmony \\
Layout hierarchy
\end{tabular}
& \begin{tabular}[t]{@{}l@{}}
Background/scene fit \\
Visual communication \\
Text readability \\
Commercial desire \\
Professional polish
\end{tabular} \\
\bottomrule
\end{tabular*}
\caption{Judge criteria: hard gates and quality dimensions.}
\label{tab:judge-criteria}
\end{table}

\paragraph{Benchmark and Evaluation.}
The benchmark contains 700 held-out products from 8 market regions and 31 top-level categories, covering 367 fine-grained category paths. Its normalized Shannon evenness is 0.833 at the top level and 0.946 at the fine-grained level. Most source images are selected through judge-score filtering and manual review. We also include lower-scoring white-background images to reduce ceiling effects.

We evaluate generated images using the rubric in Table~\ref{tab:judge-criteria}. To reduce the risk of overfitting the Gemini-2.5-Flash evaluator used during training, the benchmark evaluation is conducted with Gemini-2.5-Pro. The judge reports a weighted overall score and eleven quality scores on a 1--10 scale. We report the overall score together with professional polish, layout hierarchy, and commercial desire. We further validate the results through blind randomized pairwise human evaluation.

\FloatBarrier
\begin{figure*}[!t]
\centering
\includegraphics[width=\textwidth]{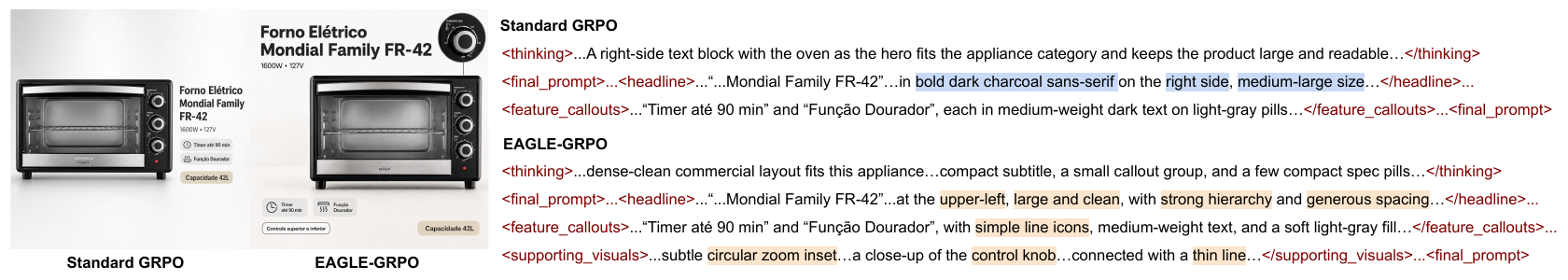}
\caption{Qualitative comparison between Standard GRPO and
\textsc{Eagle}-GRPO on one held-out oven product using GPT-Image-2.}
\label{fig:qualitative}
\end{figure*}

\begin{table}[!b]
\centering
\scriptsize
\setlength{\tabcolsep}{1.0pt}
\begin{tabular*}{\columnwidth}{@{\extracolsep{\fill}}p{0.34\columnwidth}cccc@{}}
\toprule
\textbf{Prompt writer} &
\begin{tabular}[c]{@{}c@{}}\textbf{Overall}\end{tabular} &
\begin{tabular}[c]{@{}c@{}}\textbf{Professional}\\\textbf{Polish}\end{tabular} &
\begin{tabular}[c]{@{}c@{}}\textbf{Layout}\\\textbf{Hierarchy}\end{tabular} &
\begin{tabular}[c]{@{}c@{}}\textbf{Commercial}\\\textbf{Desire}\end{tabular} \\
\midrule
\multicolumn{5}{l}{\textbf{GPT-Image-2}} \\
GPT-5.4 (Closed) & 9.205 & 9.561 & 9.443 & 8.845 \\
Gemini-2.5-Pro (Closed) & 9.077 & 9.249 & 9.459 & 8.894 \\
GPT-4o (Closed) & 8.049 & 8.704 & 8.514 & 7.634 \\
Qwen3-VL-32B (Open) & 8.914 & 9.096 & 9.237 & 8.688 \\
Qwen3-VL-32B SFT (Open) & 8.779 & 9.097 & 9.136 & 8.346 \\
Qwen3-VL-8B (Open) & 8.095 & 8.832 & 8.700 & 7.826 \\
Qwen3-VL-8B SFT (Open) & 9.056 & 9.320 & 9.556 & 8.879 \\
Standard GRPO (Trained) & 9.130 & 9.516 & 9.383 & 8.678 \\
PromptEnhancer (Trained) & 8.921 & 9.523 & 9.319 & 8.494 \\
\textsc{Eagle}-GRPO (Trained) & \textbf{9.295} & \textbf{9.634} & \textbf{9.569} & \textbf{8.971} \\
\midrule
\multicolumn{5}{l}{\textbf{FLUX.2~[klein] 9B}} \\
GPT-5.4 (Closed) & 7.963 & 8.524 & 8.408 & \textbf{6.504} \\
Gemini-2.5-Pro (Closed) & 7.035 & 7.144 & 8.017 & 6.299 \\
GPT-4o (Closed) & 5.475 & 6.084 & 6.393 & 4.679 \\
Qwen3-VL-32B (Open) & 5.084 & 5.156 & 6.036 & 4.138 \\
Qwen3-VL-32B SFT (Open) & 3.732 & 4.364 & 4.572 & 2.901 \\
Qwen3-VL-8B (Open) & 6.575 & 7.287 & 7.372 & 6.017 \\
Qwen3-VL-8B SFT (Open) & 7.306 & 7.634 & 8.156 & 6.331 \\
Standard GRPO (Trained) & 7.924 & 8.526 & 8.233 & 6.294 \\
PromptEnhancer (Trained) & 7.714 & 8.519 & 8.356 & 6.238 \\
\textsc{Eagle}-GRPO (Trained) & \textbf{7.966} & \textbf{8.551} & \textbf{8.410} & 6.408 \\
\bottomrule
\end{tabular*}
\caption{Prompt-writer performance under GPT-Image-2 and FLUX.2~[klein] 9B. All scores are on a 0–10 scale. Overall is the weighted mean across 11 quality dimensions.}
\label{tab:planner-editor-matrix}
\end{table}
\subsection{Main Results}

\paragraph{Benchmark results.}
Table~\ref{tab:planner-editor-matrix} compares different prompt writers under both image editors. \textsc{Eagle}-GRPO achieves the highest overall score in both settings. With GPT-Image-2, \textsc{Eagle}-GRPO reaches an overall score of 9.295, outperforming all compared baselines, including GPT-5.4 and Gemini-2.5-Pro. It also achieves the highest scores in professional polish, layout hierarchy, and commercial desire dimensions. With FLUX.2~[klein], \textsc{Eagle}-GRPO achieves the highest overall score of 7.966. It also obtains the best professional-polish and layout-hierarchy scores.

\paragraph{Human evaluation.}
Since the training reward and benchmark are both scored by automated VLM judges, we further validate our results with human evaluation to guard against judge-specific biases. Human evaluation further supports the benchmark results (Table~\ref{tab:human-pairwise}). Human raters prefer \textsc{Eagle}-GRPO over every evaluated comparator in both settings. In particular, against Standard GRPO, \textsc{Eagle}-GRPO achieves win rates of 62\% under GPT-Image-2 and 54\% under FLUX.2~[klein].

\begin{table}[!t]
\centering
\scriptsize
\setlength{\tabcolsep}{2pt}
\begin{tabular*}{\columnwidth}{@{\extracolsep{\fill}}>{\centering\arraybackslash}p{0.30\columnwidth}>{\centering\arraybackslash}p{0.30\columnwidth}>{\centering\arraybackslash}p{0.30\columnwidth}@{}}
\toprule
& \multicolumn{2}{c}{\textbf{Win rate of \textsc{Eagle}-GRPO}} \\
\textbf{\textsc{Eagle}-GRPO vs.} & \textbf{GPT-Image-2} & \textbf{FLUX-2-Klein} \\
\midrule
GPT-5.4 & 52\% & 52\% \\
Gemini-2.5-Pro & 51\% & 53\% \\
GPT-4o & 82\% & 87\% \\
Qwen3-VL-32B & 75\% & 59\% \\
Qwen3-VL-32B SFT & 56\% & 57\% \\
Qwen3-VL-8B & 79\% & 56\% \\
Qwen3-VL-8B SFT & 56\% & 59\% \\
PromptEnhancer & 58\% & 57\% \\
Standard GRPO & 62\% & 54\% \\
\bottomrule
\end{tabular*}
\caption{Blind pairwise human evaluation. Each cell reports the win rate of \textsc{Eagle}-GRPO over the listed baseline. }
\label{tab:human-pairwise}
\end{table}

\section{Ablation Studies}
\label{sec:ablation}

We design the ablation studies to isolate which component drives the observed improvements. In particular, we examine whether the gains arise from element-level credit assignment itself and how this design compares with state-of-the-art alternatives. We conduct two controlled comparisons. 

\subsection{Element-Level versus Sequence-Level Credit}
To isolate whether the gains arise specifically from element-level credit assignment, we compare \textsc{Eagle}-GRPO with a matched Standard GRPO baseline under the same training setup. For each image editor, we use the same SFT initialization, rollout budget, reward function, optimizer settings, and 200-update schedule. The only difference is credit assignment. Standard GRPO applies one rollout-level advantage to all prompt tokens, whereas \textsc{Eagle}-GRPO redistributes the advantage across prompt elements.

Element-aware credit improves the benchmark score in both cases. The overall score improves from 9.130 to 9.295 with GPT-Image-2 and from 7.924 to 7.966 with FLUX (Table~\ref{tab:planner-editor-matrix}). Human evaluation shows the same trend (Table~\ref{tab:human-pairwise}).




\paragraph{Sustained Gains and Image Quality.} Beyond final image fidelity, we examine whether element-level credit supports more sustained optimization. Figure~\ref{fig:training-dynamics} reports validation reward on a fixed 500-item set every 10 updates. Standard GRPO improves earlier but later plateaus or declines. In contrast, \textsc{Eagle}-GRPO continues improving and reaches a higher validation reward under both image editors.


\begin{figure}[h]
\centering
\includegraphics[width=\linewidth]{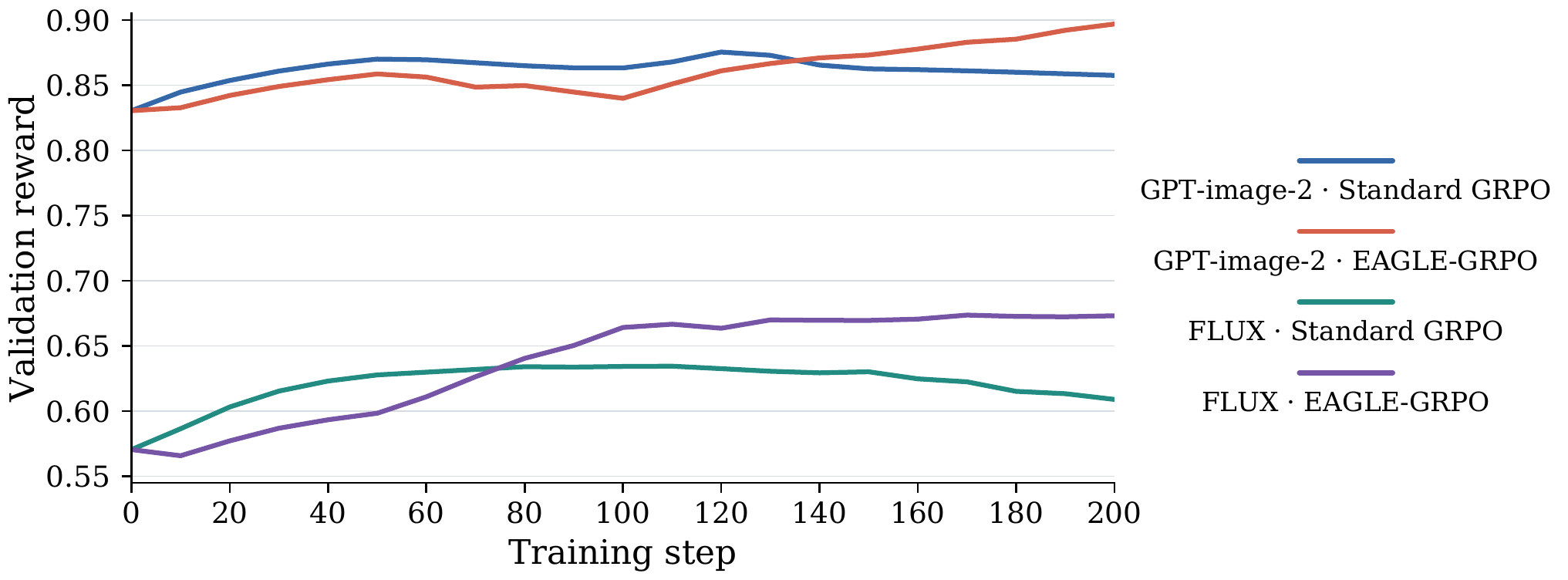}
\caption{Validation reward, evaluated every 10 updates.}
\label{fig:training-dynamics}
\end{figure}

\paragraph{Qualitative comparison.}
Figure~\ref{fig:qualitative} compares Standard GRPO and
\textsc{Eagle}-GRPO on one held-out product using GPT-Image-2. Both policies identify the main product features. \textsc{Eagle}-GRPO strengthens the \elementname{headline} through a larger upper-left placement with clearer hierarchy, and uses \elementname{supporting\_visuals} to request a close-up inset of the control knob. The resulting image presents the product information more clearly while preserving the overall visual hierarchy.

\subsection{Element-Level Credit versus Fine-Grained Scoring}

To determine whether fine-grained element scoring alone is sufficient when the policy update remains sequence-level, we compare \textsc{Eagle}-GRPO with an adaptation of PromptEnhancer~\cite{wang2025promptenhancer}, a recent state-of-the-art prompt-rewriting method based on chain-of-thought rewriting and fine-grained reward. We adapt its design to our e-commerce image-editing setting. The model is initialized with SFT, following the chain-of-thought prompt-rewriting format of the original method. During GRPO, Gemini-2.5-Flash scores each schema element from 0 to 10, and the averaged score is used as a sequence-level reward. This preserves fine-grained scoring while excluding element-level credit assignment.

\textsc{Eagle}-GRPO achieves higher benchmark scores under both editors: 9.295 versus 8.921 with GPT-Image-2, and 7.966 versus 7.714 with FLUX (Table~\ref{tab:planner-editor-matrix}). Human raters also prefer \textsc{Eagle}-GRPO in 58\% and 57\% of comparisons (Table~\ref{tab:human-pairwise}). These results reassure that the policy update also benefits from element-level credit assignment.

\section{Element-Level Diagnostics}
\label{sec:case-study}
The following analyzes use the GPT-image-2 run. Its stronger prompt adherence makes the link between structured prompt decisions and rendered content easier
to inspect.

\subsection{Rollout-Specific Element Credit}
Figure~\ref{fig:mechanism-credit} visualizes the credits computed for 10 samples in a rollout group. In \textsc{Eagle}-GRPO (Figure~\ref{fig:mechanism-credit}(b)), colors vary across rows within the same rollout, showing that different elements receive different credit even when the overall image reward is the same. This demonstrates that \textsc{Eagle}-GRPO assigns different credits to the same element across rollouts, and these credits in turn produce different final advantages for the corresponding element tokens.

\begin{figure}[h]
\centering
\includegraphics[width=\columnwidth]{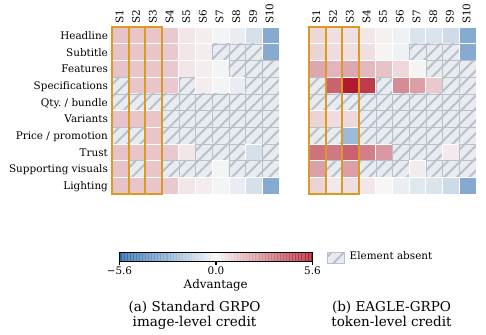}
\caption{Comparison of sample-level credit in Standard GRPO and element-specific credit in \textsc{Eagle}-GRPO. Each column is one rollout; each row is a prompt element. Color indicates the credit assigned to that element in that rollout. }
\label{fig:mechanism-credit}
\end{figure}
\FloatBarrier

\subsection{Category-Level Patterns}
We aggregate optional-element statistics by product category to examine how often each element is selected and how its credit relates to image reward. Figure~\ref{fig:category-element} reports these patterns across seven distinct categories and eight optional elements.

The clearest patterns appear for \elementname{specification} in Home \& Living and \elementname{features} in Health and Sports \& Outdoors (48\% presence, $r=0.78$). In contrast, \elementname{subtitle} is selected more frequently in Home \& Living but has a weaker reward correlation. This shows that frequent selection does not necessarily imply stronger alignment with image quality. These correlations describe reward allocation within \textsc{Eagle}-GRPO, not the causal effect of adding an element.

\begin{figure}[!t]
\centering
\includegraphics[width=\columnwidth]{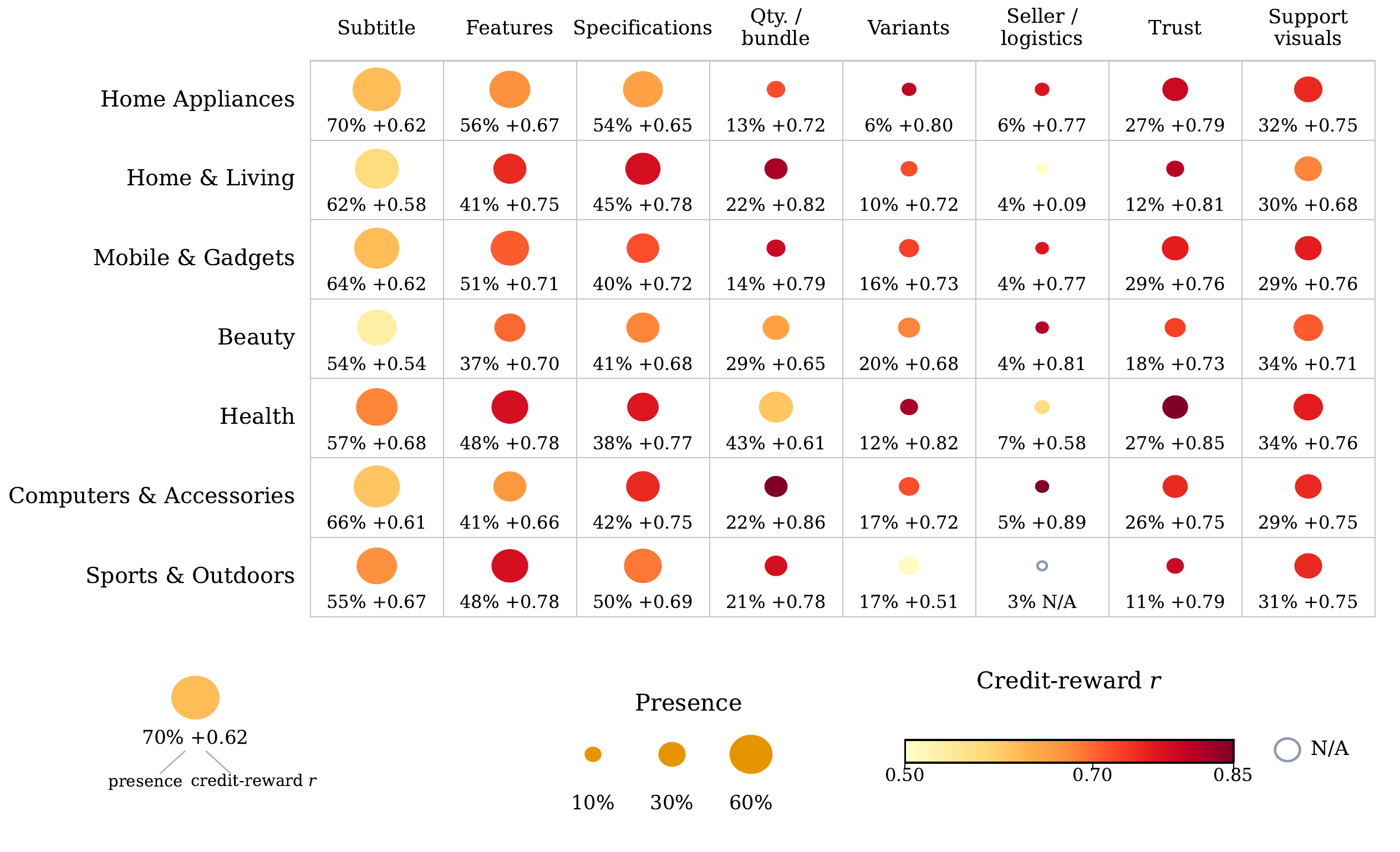}
\caption{Category-level diagnostics for eight optional prompt elements. Marker size indicates selection frequency, darker colors indicate stronger credit--reward correlation.}
\label{fig:category-element}
\end{figure}

\subsection{Element Evolution Across Checkpoints}
To examine how element-level credit affects prompt behavior, we track \elementname{promotion\_information} for one fixed product, sampling 20 prompts from the SFT model and five GRPO checkpoints. Its usage rate first rises at step 40, reflecting early exploration, but the element receives negative kernel credit despite the slightly higher raw reward of prompts that contain it, suggesting that the apparent reward gain instead comes from other co-occurring elements. Continued negative credit then drives its usage to zero by step 200.


Figure~\ref{fig:tw-prompt-evolution} shows how these changes are reflected in the generated images. The step-40 image includes a generic promotional message meaning ''storewide free shipping from \$99,'' whereas the step-200 output adds a control-panel close-up and grounded product details meaning ''high-temperature lock'', ''leak detection with automatic water shutoff'', etc. This illustrates a shift from generic promotion toward product-specific information and richer visual support.

\section{Conclusion}
We presented \textsc{Eagle}-GRPO, a kernel-based method that decomposes the centered image reward into per-element credits from variation already present in a standard GRPO rollout group. 
These credits are mapped back to their corresponding token spans while preserving per-sample 
advantage. 
We proved that the \textsc{Eagle}-GRPO beats Standard GRPO and competitive prompt writers across GPT-Image-2 and FLUX.2.
Element-level and longitudinal analysis further reveal which prompt decisions are reinforced or suppressed, making reward-driven optimization more targeted and interpretable.

\paragraph{Limitations and Future Directions.}
In the setting of this work, the decomposition fits 18 element kernels, each capturing element presence and semantic variation. Different product categories emphasize different elements, so a fixed 18-kernel design leaves some dimensions sparse. Future work could address this by tailoring the kernel to each category, tuning 
$\lambda$ per category, or extending the decomposition to pairwise interaction kernels.


\bibliography{main}

\end{document}